\documentclass[10pt,journal,compsoc]{IEEEtran}

\newif\ifpeerreview
\peerreviewfalse

\usepackage[nocompress]{cite}
\usepackage{url}
\usepackage{amsmath,amssymb,amsthm,graphicx}
\usepackage{lipsum}

\usepackage[switch]{lineno}
\usepackage{subcaption}
\usepackage{tikz}

\newtheorem{theorem}{Theorem}[section]
\newtheorem{lemma}[theorem]{Lemma}

\newtheorem{example}[theorem]{Example}
\newtheorem{proposition}[theorem]{Proposition}

\newtheorem{definition}[theorem]{Definition}

\newtheorem{observation}[theorem]{Observation}

\newcommand{\E}{\mathbb{E}}
\newcommand{\N}{\mathbb{N}}

\newcommand{\R}{\mathbb{R}}

\newcommand{\dx}{\mathrm{d}}

\newcommand{\lPa}{\left (}
\newcommand{\rPa}{\right )}

\newcommand{\lBc}{\left \{}
\newcommand{\rBc}{\right \}}

\newcommand{\In}{\mathbf{1}}

\DeclareMathOperator{\tr}{tr}

\newcommand{\paperID}{36}
\title{Multiscale~Super~Resolution~without~Image~Priors}

\author{Daniel Fu, Gabby Litterio, Pedro Felzenszwalb, Rashid Zia
\IEEEcompsocitemizethanks{\IEEEcompsocthanksitem Division of Applied Mathematics (Fu) and School of Engineering \\ (Litterio, Felzenszwalb, Zia), Brown University, Providence, RI, 02906.  \protect\\
E-mail: daniel\_fu1@brown.edu, gabby\_litterio@brown.edu, \\ pedro\_felzenszwalb@brown.edu, rashid\_zia@brown.edu}}

\begin{document}

\IEEEtitleabstractindextext{%
\begin{abstract}
We address the ambiguities in the super-resolution problem under
translation. We demonstrate that combinations of low-resolution images
at different scales can be used to make the super-resolution problem
well posed. Such differences in scale can be achieved using sensors
with different pixel sizes (as demonstrated here) or by varying the
effective pixel size through changes in optical magnification (e.g.,
using a zoom lens).  We show that images acquired with pairwise
coprime pixel sizes lead to a system with a stable inverse, and
furthermore, that super-resolution images can be reconstructed
efficiently using Fourier domain techniques or iterative least squares
methods. Our mathematical analysis provides an expression for the
expected error of the least squares reconstruction for large signals
assuming i.i.d. noise that elucidates the noise-resolution
tradeoff. These results are validated through both one- and
two-dimensional experiments that leverage charge-coupled device (CCD)
hardware binning to explore reconstructions over a large range of
effective pixel sizes. Finally, two-dimensional reconstructions for a
series of targets are used to demonstrate the advantages of multiscale
super-resolution, and implications of these results for common imaging
systems are discussed.
\end{abstract}

}

\maketitle

\IEEEraisesectionheading{
  \section{Introduction}\label{sec:introduction}
}

\IEEEPARstart{T}{he} goal of super-resolution imaging is to
reconstruct, or estimate, a high-resolution image from one or more
low-resolution measurements. Previous work has discussed the
ambiguities inherent in super resolution under translation
\cite{baker-2002,lin-2004,litterio-2025}, which in turn have inspired
the use of image priors to address missing information. For example,
\cite{baker-2002} showed the limitation of a quadratic smoothness
prior and proposed a recognition-based approach where high-resolution
features are hallucinated from low-resolution measurements by
leveraging training data. Subsequent work explored the use of large
training datasets \cite{hayes-sun2012-super,dong2015-deep} to estimate
high-resolution detail from a single low-resolution image.  While
often effective, such methods rely on assumptions about the underlying
image distribution, require large training datasets in specific
domains, and can struggle to generalize in new settings.  Recently,
\cite{litterio-2025} showed that total variation (TV) can be used to
reconstruct sparse images without the need for training data, but such
priors will fail for sufficiently complex scenes.  Other examples of
image priors that can (and have) been used for super resolution
include methods based on self-similarity
\cite{glasner_super-resolution_2009} and implicit priors based on
neural representations \cite{xie_neural_2022}.

In this paper, rather than relying on image priors, we show that the
ambiguities in super resolution under translation can be resolved
using data collected at multiple scales or effective
resolutions. Here, scale refers to the effective pixel size, which
depends on both the physical size of the image sensor's pixels as well
as the magnification of the imaging system.  In practice, scale
variation can be achieved through several mechanisms, such as changing
the optical magnification of a zoom lens, swapping intermediate lenses
in the optical system of a microscope, using sensors with multiple
pixel sizes, or using devices that have multiple independent cameras
with different effective resolutions.

\begin{figure}
  \captionsetup{justification=centering}
  \centering
  \begin{subfigure}{0.27\columnwidth}
    \centering
    \caption*{Static Image\\221~$\mu$m Pixels}
      \includegraphics[height=1\textwidth]{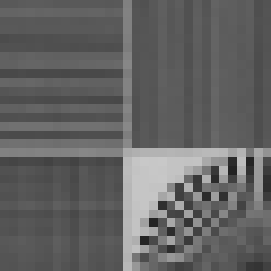} 
    \end{subfigure}
  \begin{subfigure}{0.27\columnwidth}
    \centering
  \caption*{Static Image\\234~$\mu$m Pixels}
      \includegraphics[height=1\textwidth]{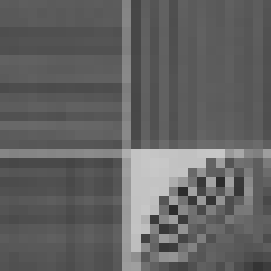} 
  \end{subfigure}
  \begin{subfigure}{0.27\columnwidth}
    \centering
    \caption*{Static Image\\247~$\mu$m Pixels}
    \includegraphics[height=1\textwidth]{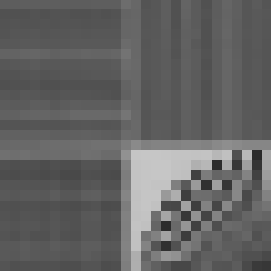} 
  \end{subfigure}

  \vspace{0.1cm}

    \begin{subfigure}{0.27\columnwidth}
    \centering
    \includegraphics[height=1\textwidth]{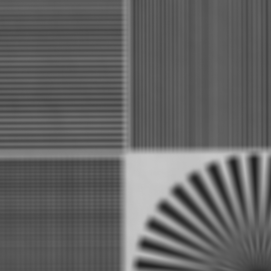} 
  \vspace{-0.55cm}
  \caption*{Interlaced Data\\221$\mu$m Pixels}
  \end{subfigure}
  \begin{subfigure}{0.27\columnwidth}
    \centering
    \includegraphics[height=1\textwidth]{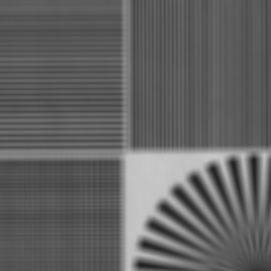} 
  \vspace{-0.55cm}
    \caption*{Interlaced Data\\234$\mu$m Pixels}
  \end{subfigure}
  \begin{subfigure}{0.27\columnwidth}
    \centering
    \includegraphics[height=1\textwidth]{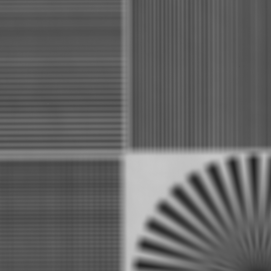}
  \vspace{-0.55cm}
  \caption*{Interlaced Data\\247$\mu$m Pixels}
  \end{subfigure}


  \begin{subfigure}{0.4\columnwidth}
    \centering
    \includegraphics[height=0.95\textwidth]{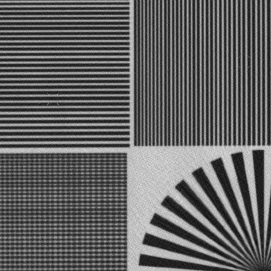} 
 \vspace{-0.1cm}
     \caption*{Multiscale Reconstruction}
  \end{subfigure}
  \hspace{0.3cm}
  \begin{subfigure}{0.4\columnwidth}
    \centering
    \includegraphics[height=0.95\textwidth]{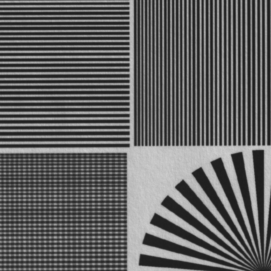} 
 \vspace{-0.1cm}
     \caption*{Static Image 13~$\mu$m Pixels}
  \end{subfigure}
  
\vspace{-0.1cm}
\captionsetup{justification=raggedright}
\caption{Multiscale super-resolution demo using 200+~$\mu$m pixels to reconstruct images comparable to 13~$\mu$m pixel.}
\label{fig:200micronDemo}
\end{figure}

Our approach is motivated by the observation that data collected at
different scales (effective pixel sizes) provide complementary
information in the Fourier domain.  While data collected with a single
low-resolution sensor and magnification cannot resolve all of the
frequency components of an image, combining measurements across
appropriately chosen scales fills in the gaps and eliminates the
ambiguity in the inverse problem.

Figure~\ref{fig:200micronDemo} provides an illustrative example of how
measurements acquired at three carefully chosen scales can enable
super-resolution in two dimensions. For this demonstration, and to
enable systematic study of super resolution under translation, we
leverage charge-coupled device (CCD) hardware binning prior to data
readout. The top row shows individual images captured via hardware
binning with effective pixel sizes of 221, 234, and 247~$\mu$m
respectively. These sizes were chosen as coprime integer multiples
(17, 18, and 19) of 13~$\mu$m. The second row shows the result of
interlacing multiple low-resolution images together. The last row
shows a reconstruction computed without image priors, using the
Fourier domain method presented below, next to a "ground-truth'' image
captured using 13~$\mu$m pixels (without binning) for
comparison. Despite using sensors with effective pixel sizes 17-19x
larger, the multiscale reconstruction based on the 200+~$\mu$m pixel
data recovers the fine structure of the 13~$\mu$m pixel image.

Most generally, for a signal $u$ in $d$-dimensions, our main
theoretical result establishes that measurements with $d+1$ pairwise
coprime integer pixel sizes are sufficient for stable reconstruction
using imaging constraints (Theorem~\ref{thm:condition_number}).
Furthermore, we provide an expression for the expected error of the
least squares reconstruction for large signals assuming i.i.d. noise
that elucidates the noise-resolution tradeoff
(Theorem~\ref{thm:cyclic_trace} and Observation~\ref{obs:tradeoff}).
These results are validated through both one- and two-dimensional
experiments that leverage CCD hardware binning to explore
reconstructions over a large range of effective pixel sizes.

We demonstrate the advantages of using multiple scales to reconstruct
two-dimensional images by empirically comparing the reconstructions
obtained with one, two, or three scales.  Implications of these
results for common imaging systems are discussed in the conclusion.

\section{Imaging Model} 
\label{sec:model}

Let $E: \R^{2} \rightarrow \R$ denote a continuous irradiance field on
the imaging plane of a camera.  An idealized high-resolution sensor
$H$ with square pixels of size $\Delta$ records a discrete image $I$
by spatially integrating $E$ over each pixel,
\[
I[i, j] = \int_{0}^{\Delta} \int_{0}^{\Delta} E(x + i\Delta, y +
j\Delta) \, \dx{x} \dx{y}.
\]
Since the irradiance field $E$ is diffraction-limited, the ``true''
resolution (resolving power) of the image $I$ is a combination of the
optical system and the pixel size. This is why the design of a
high-resolution imaging systems often begins with a pixel size
$\Delta$ that is smaller than the diffraction-limited optical
resolution.

Throughout this paper we consider the reconstruction of the image $I$
from measurements obtained using low-resolution sensors with pixel
sizes significantly larger than $\Delta$.  Thus, we are focused on
improving the sensor resolution and working in regimes where the
diffraction-limited image has higher resolution than our sensors.  In
other words, we are focused on sensor super-resolution imaging, not
super-resolution below the diffraction limit, which is a distinct
field in optical microscopy.

To define a tractable model for analysis, we consider sensors with
pixel sizes that are integer multiples of $\Delta$. This simplifies
the analysis, and it also represents a robust model for CCD cameras in
which photoelectrons accumulated within individual pixels can be
binned into a contiguous rectangular form prior to readout. Such CCD
hardware binning is regularly used to improve noise statistics for
low-signal measurement, and we will leverage binning to explore a
range of effective pixel sizes.

Suppose we have a low-resolution sensor $L$ with pixel size $k\Delta$
for a positive integer $k$.  When $L$ is properly aligned with $H$,
each low-resolution pixel measures the sum of $k^2$ high-resolution
pixels.  By translating $L$ within the imaging plane of the camera (or
equivalently by translating the camera) and then recording a series of
low-resolution images, we can measure the convolution of $I$ with a $k
\times k$ box filter $b_k$. In practice, measuring the convolution
involves capturing $k^2$ low-resolution images in a grid of sub-pixel
locations and interlacing the result as shown in
Figure~\ref{fig:interlace}.

Note that this interlacing process can be readily performed by
commercially-available cameras with pixel-shift sensor technology. As
demonstrated by \cite{litterio-2025}, it is also possible to use
camera motion with stepper motors and other low-cost translation
stages to interlace data. We have also captured the equivalent of
interlaced data by using visual landmarks (AprilTags) to localize
low-resolution images to sub-pixel locations.

\begin{figure}
  \captionsetup{justification=centering}
  \centering
  \begin{subfigure}{0.3\columnwidth}
    \centering
  \caption*{High-Res Target}
    \includegraphics[height=1\textwidth]{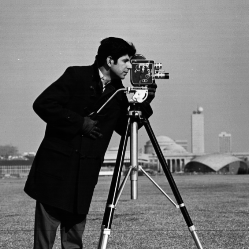} 
    \end{subfigure}
  \begin{subfigure}{0.3\columnwidth}
    \centering
  \caption*{Low-Res Images}
      \includegraphics[height=1\textwidth]{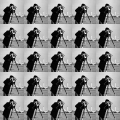} 
  \end{subfigure}
  \begin{subfigure}{0.3\columnwidth}
    \centering
    \caption*{Interlaced Data}
    \includegraphics[height=1\textwidth]{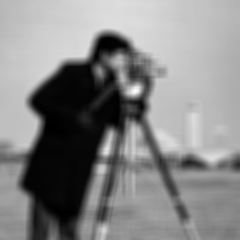} 
  \end{subfigure}
\vspace{-0.1cm}
\captionsetup{justification=raggedright}
\caption{Capturing $(I \otimes b_k)$ with a low-resolution sensor
  involves interlacing multiple low-resolution images.}
\label{fig:interlace}
\end{figure}

For the stability analysis below, we further assume the measurements
obtained with the low-resolution sensor $L$ are corrupted by noise, \(
J = (I \otimes b_k) + \epsilon.  \) For analytical convenience, we
assume that $\epsilon$ is Gaussian noise that is independent across
pixels and across measurements.  We remark that this assumption is a
mathematical simplification, and is not a claim regarding the exact
physical noise mechanism. (It is well understood that there are
several sources of noise when working with a CCD
\cite{janesick-2007}). Nevertheless, this assumption retains the core
structure of the reconstruction problem and leads to a tractable
analysis, which we will show is in good agreement with empirical data.

\subsection{Ambiguities}

As discussed in \cite{baker-2002, lin-2004, deconv24}, the problem of
super resolution under translation, which involves recovering $I$ from
$J$, is ill-posed.  Figure~\ref{fig:ill-posed} shows a simple example
of two different one-dimensional signals that are indistinguishable
after convolution with a one-dimensional box.  The nullspace of
convolution with a box is characterized in \cite{deconv24}.

\begin{figure}
    \centering
    \includegraphics[width=\columnwidth]{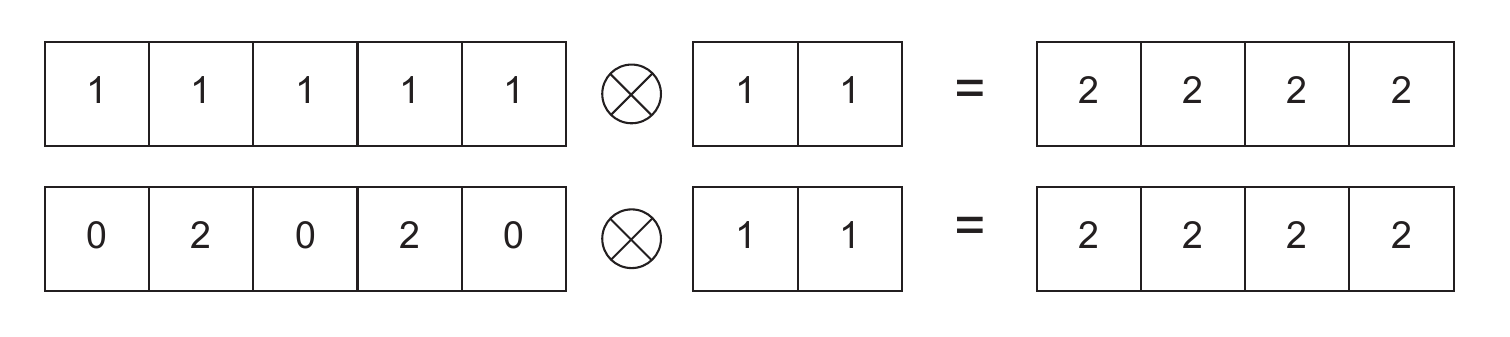}
    \caption{Two signals that are indistinguishable after convolution with a box of size two.}
    \label{fig:ill-posed}
\end{figure}

The example in Figure~\ref{fig:ill-posed} illustrates the case where
the measurements correspond to the ``valid'' convolution of a signal
of length $n$ with a box of length $k$, resulting in $n-k+1$
measurements.  We note that it is also possible to introduce a mask in
the sensor to measure the ``full'' convolution with a box, resulting
in $n+k-1$ measurements. (Such masks are indeed a standard part of
many image sensors.) The full convolution is technically invertible,
but the condition number of the system grows unbounded as the signal
length $n$ increases.  This is similar to the situation of cyclic
convolution discussed in Section~\ref{sec:stability}.

\subsection{Multiple scales (pixel sizes)}
\label{sec:multiple_scales}

To resolve the ambiguities discussed above, we propose leveraging
measurements at more than one resolution.  In particular, we will show
that a $d$-dimensional signal $u$ can be effectively recovered from
the convolution of $u$ with $d+1$ boxes of different sizes, as long as
the sizes are pairwise coprime ($a$ and $b$ are coprime if
$\gcd(a,b)=1$).

For two-dimensional images, we can estimate $I$ as long as we have
measurements with three different pixel sizes, even when they are all
relatively large. For example, using a regular camera with a zoom
lens, we can change the effective pixel size of a sensor.  In this
case, in addition to obtaining interlaced images by moving the sensor
or the camera, to obtain a magnification factor of $m$, it suffices to
use 3 zoom settings with magnifications $m/k_1, m/k_2, m/k_3$, where
$k_1, k_2, k_3$ are coprime integers.  For instance, a magnification
of 100x can be obtained from pictures taken with magnifications
$100/9\approx11.11$, $100/10=10$, $100/11\approx 9.09$. Likewise, a
discrete set of 3 effective zooms can also be implemented in a
microscope setting with different objectives and tube lens
combinations.

\section{Resolving Super-Resolution Ambiguities}
\label{sec:ambiguity}

As discussed above and illustrated in Figure~\ref{fig:ill-posed}, even
in the absence of noise, there is ambiguity in recovering a signal $u$
from $y = u \otimes b$ for a single box filter $b$.

In this section, we illustrate how this ambiguity can be resolved by
incorporating measurements using multiple effective pixel sizes.  In
particular we illustrate directly in the spatial domain how
intersection patterns of boxes of different sizes enable the recovery
of a value in $u$ from a small number of low-resolution measurements.

\subsection{1D and Bezout's identity}

Let $u$ be a one-dimensional signal of length $n$.  Let $b_k$ denote a
box of size $k$ and $y_{k} = u \otimes b_{k}$.

As a simple example, suppose that we observe both $y_{k}$ and $y_{k +
  1}$ for any positive integer $k$.  Note that $y_{k}$ gives sums of
$k$ consecutive values in $u$ while $y_{k+1}$ gives sums of $k+1$
consecutive values.  Therefore, as shown in the left-hand-side of
Figure~\ref{fig:1dboxes}, the original signal at position $a$ can be
recovered by taking differences of two measurements,
\begin{equation}
u(a) = y_{k + 1}(a) - y_{k}(a + 1).
\label{eqn:dim_one_special_case}
\end{equation}

\begin{figure}
\leavevmode\par\medskip
\begin{center}
\vspace{-0.5cm}
\begin{tikzpicture}[scale=0.4]
    \def\k{5}
    \pgfmathtruncatemacro{\N}{10}
    
    \fill[red,fill opacity=0.5] (0,2.5) rectangle (\k+1,3.5);
    \draw[line width=2pt,red] (0,2.5) rectangle (\k+1,3.5);

    \draw[<->, thick] (0,3.75) -- (\k+1,3.75);
    \node[font=\large] at (0.5*\k+0.5,4.25) {$k+1$};

    \fill[blue,fill opacity=0.5] (1,2) rectangle (\k+1,3);
    \draw[line width=2pt,blue] (1,2) rectangle (\k+1,3);
    
    \draw[<->, thick] (1,2-0.25) -- (\k+1,2-0.25);
    \node[font=\large] at (0.5*\k+1,2-0.75) {$k$};
    
    \fill[red,fill opacity=0.5] (8+0,-1+4.5) rectangle (8+\k,-2+4.5);
    \draw[line width=2pt,red] (8+0,-1+4.5) rectangle (8+\k,-2+4.5);

    \draw[<->, thick] (8+0,-.75+4.5) -- (8+\k,-.75+4.5);
    \node[font=\large] at (8+0.5*\k,-.25+4.5) {$k_2$};

    \fill[red,fill opacity=0.5] (8+\k,-1+4.5) rectangle (8+2*\k,-2+4.5);
    \draw[line width=2pt,red] (8+\k,-1+4.5) rectangle (8+2*\k,-2+4.5);

    \draw[<->, thick] (8+\k,-.75+4.5) -- (8+2*\k,-.75+4.5);
    \node[font=\large] at (8+1.5*\k,-.25+4.5) {$k_2$};
    
    \fill[blue,fill opacity=0.5] (8+1,-1.5+4.5) rectangle (8+\k-1,-2.5+4.5);
    \draw[line width=2pt,blue] (8+1,-1.5+4.5) rectangle (8+\k-1,-2.5+4.5);

    \draw[<->, thick] (8+1,-2.5-0.25+4.5) -- (8+\k-1,-2.5-0.25+4.5);
    \node[font=\large] at (8+0.5*\k,-2.5-0.75+4.5-.1) {$k_1$};

    \fill[blue,fill opacity=0.5] (8+\k-1,-1.5+4.5) rectangle (8+2*\k-3,-2.5+4.5);
    \draw[line width=2pt,blue] (8+\k-1,-1.5+4.5) rectangle (8+2*\k-3,-2.5+4.5);

    \draw[<->, thick] (8+\k-1,-2.5-0.25+4.5) -- (8+2*\k-3,-2.5-0.25+4.5);
    \node[font=\large] at (8+0.5*\k+\k-2,-2.5-0.75+4.5-.1) {$k_1$};

    \fill[blue,fill opacity=0.5] (8+2*\k-3,-1.5+4.5) rectangle (8+3*\k-5,-2.5+4.5);
    \draw[line width=2pt,blue] (8+2*\k-3,-1.5+4.5) rectangle (8+3*\k-5,-2.5+4.5);

    \draw[<->, thick] (8+2*\k-3,-2.5-0.25+4.5) -- (8+3*\k-5,-2.5-0.25+4.5);
    \node[font=\large] at (8+0.5*\k+2*\k-4,-2.5-0.75+4.5-.1) {$k_1$};
\end{tikzpicture}
\end{center}
\vspace{-0.5cm}
\caption{Examples showing how single-pixel values can be calculated by subtraction of different box sums in 1D.}
\label{fig:1dboxes}
\end{figure}
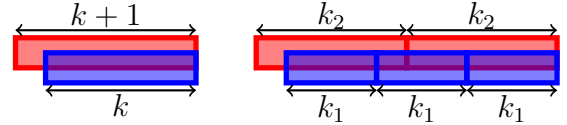

Thus, together the pair $y_{k}$ and $y_{k+1}$ contains enough
information to recover the original signal.  Moreover,
Eq.~\eqref{eqn:dim_one_special_case} is local and only two
measurements are combined without scaling their magnitudes, so this
approach leads to a stable reconstruction in the sense that
measurement errors are not significantly amplified.

More generally, any two coprime box sizes suffice to obtain a stable
reconstruction of $u$ in one dimension. Let $k_1$ and $k_2$ be two
coprime positive integers.  By Bezout's identity (see,
e.g. \cite{jones-1998}), and relabeling if necessary, there are
positive integers $m_{1}, m_{2}$ such that,
\[
    m_2k_2 = m_1k_1 + 1. 
\]

By summing over consecutive windows of size $k_j$, we can obtain
$y_{m_jk_j}$ from $y_{k_j}$.  We can then recover $u$ using
Eq.~\eqref{eqn:dim_one_special_case} with $k=m_1k_1$ and
$k+1=m_2k_2$. Further, we can choose $m_{1} \leq k_{2}$ and $m_{2}
\leq k_{1}$, so that this reconstruction is local and stable.
Figure~\ref{fig:1dboxes} illustrates examples where $m_1=3$ and
$m_2=2$ (when $k_2=k_1+2$).

\subsection{Invertibility}

The following result characterizes the invertibility when sensors
measure the valid part of the convolution of a signal with a box,
which is a subset of the full or cyclic convolution.  Invertibility in
this case can be used to derive local reconstruction methods.

\begin{theorem}
Suppose $n \geq k_{1} + k_{2} - 1$ and let $u \otimes b$ denote the
valid part of the convolution.  Then the linear map,
\begin{equation}
    u \rightarrow (u \otimes b_{k_{1}}, u \otimes b_{k_{2}})
\label{eqn:linear_map_dim_one}
\end{equation}
is invertible if and only if $\gcd(k_{1}, k_{2}) = 1$.
\end{theorem}

\begin{proof}
Suppose first that $k_{1}$ and $k_{2}$ are not coprime with $r =
\gcd(k_{1}, k_{2})$. Any non-zero $r$-periodic signal $x$ with
\[
    \sum_{a = 0}^{r - 1} x(a) = 0
\]
lies in the kernel of the map~\eqref{eqn:linear_map_dim_one}. Since $r
\geq 2$, the set of such signals is non-trivial
and~\eqref{eqn:linear_map_dim_one} cannot be invertible.

Now assume that $k_{1}$ and $k_{2}$ are coprime. By \cite[Proposition
  1]{deconv24}, any element $x$ in the kernel of the
map~\eqref{eqn:linear_map_dim_one} is necessarily both
$k_{1}$-periodic and $k_{2}$-periodic. Hence, the Fine-Wilf theorem
implies that $x$ is $1$-periodic and therefore constant.  But then, $x
\otimes b_k = 0$ implies $x=0$.  It follows that the kernel
of~\eqref{eqn:linear_map_dim_one} is trivial, i.e. the map is
invertible.
\end{proof}

\subsection{2D and box packing} \label{subsec:2D_box_packing}

The two-dimensional case follows the same basic principle as the
one-dimensional case, but with greater ambiguity.

Indeed, two box sizes do not suffice for exact recovery in dimension
two. This is because the two-dimensional box is separable and
decomposes as the product of horizontal and vertical one-dimensional
boxes. As a result, one can construct multiple signals which are
indistinguishable to two coprime box sizes simultaneously.

Let $x_1$ and $x_2$ be one-dimensional signals in the null space of
convolution with a box of size $k_1$ and $k_2$ respectively.  Then,
the outer product $x_{1}x_{2}^{\intercal}$, which we interpret as a
two-dimensional signal, lies in the null space of convolution with
two-dimensional boxes of size $k_{1}$ \emph{and} $k_2$.  The same is
true for linear combinations of such outer products.

\begin{example}
Let $x_{1} = (1,-1,1,-1)^{\intercal}$, $x_{2} = (1,-1,0,1)^{\intercal}$ and consider
\[
    x = \In + x_{1} x_{2}^{\intercal} = \begin{pmatrix}
        2 & 0 & 1 & 2 \\
        0 & 2 & 1 & 0 \\
        2 & 0 & 1 & 2 \\
        0 & 2 & 1 & 0
    \end{pmatrix}
\]
Every $2 \times 2$ window of $x$ sums to $4$ and every $3 \times 3$
window sums to $9$ so that, after convolving with boxes of size $2$
and $3$, this signal is indistinguishable from the constant $1$'s
signal.
\end{example}
Thus, in two-dimensions, at least three box sizes are required to
reconstruct a signal.

Figure~\ref{fig:2dboxes} illustrates a key example where three box
sizes are in fact sufficient.  In this case we reconstruct $u$ from
$y_k$, $y_{k+1}$, and $y_{2k+1}$ for any positive integer $k$.  The
packing of boxes shown in Figure~\ref{fig:2dboxes} illustrates how one
individual pixel can be recovered from sums in boxes of size $k$,
$k+1$ and $2k+1$.  This leads to the expression,
\begin{align}
u(a,b) & = y_{k}(a+1,b-k) + y_{k}(a-k,b+1) + y_{k+1}(a,b) \notag \\ & + y_{k+1}(a-k,b-k) - y_{2k+1}(a-k,b-k).
\label{eqn:dim_two_special_case}
\end{align}

As in the one-dimensional case, this combines a small number of
measurements without scaling magnitudes and leads to a stable local
reconstruction.

\begin{figure}
\label{ex:box_sufficiency_in_dim_two}
\leavevmode\par\medskip
\begin{center}
\vspace{-0.5cm}
\begin{tikzpicture}[scale=0.5]
    \def\k{3}
    \pgfmathtruncatemacro{\N}{2*\k+1}

    \draw[line width=4pt,color=blue] (0,0) rectangle (\N,\N);

    \fill[red!35] (0,0) rectangle (\k,\k);
    \draw[line width=1pt] (0,0) rectangle (\k,\k);

    \fill[red!35] ({\k+1},{\k+1}) rectangle (\N,\N);
    \draw[line width=1pt] ({\k+1},{\k+1}) rectangle (\N,\N);

    \fill[green!15] (0,\k) rectangle ({\k+1},\N);
    \draw[line width=1pt] (0,\k) rectangle ({\k+1},\N);

    \fill[green!15] (\k,0) rectangle (\N,{\k+1});
    \draw[line width=1pt] (\k,0) rectangle (\N,{\k+1});

    \fill[black!55] (\k,\k) rectangle ({\k+1},{\k+1});
    \draw[line width=1pt] (\k,\k) rectangle ({\k+1},{\k+1});

    \node[font=\Large] at ({0.5*\k},{0.5*\k}) {$A$};
    \node[font=\Large] at ({\k+1+0.5*\k},{\k+1+0.5*\k}) {$B$};
    \node[font=\Large] at ({0.5*(\k+1)},{\k+0.5*(\k+1)}) {$C$};
    \node[font=\Large] at ({\k+0.5*(\k+1)},{0.5*(\k+1)}) {$D$};
    \node[font=\Large, text=white] at ({\k+0.5},{\k+0.5}) {$F$};

    \node[font=\Large,color=blue] at ({\N+0.6},{0.5*\N}) {$E$};

    \draw[<->, thick] (-0.5,0) -- (-0.5,\N);
    \node[font=\large, rotate=90] at (-1,{0.5*\N}) {$2k+1$};

    \draw[<->, thick] ({\k+1},\N+0.45) -- (\N,\N+0.45);
    \node[font=\large] at ({\k+1+0.5*\k},\N+0.95) {$k$};

    \draw[<->, thick] (\k,-0.5) -- (\N,-0.5);
    \node[font=\large] at ({\k+0.5*(\k+1)},-1) {$k+1$};

\end{tikzpicture}
\end{center}
\vspace{-0.5cm}
\caption{A 2D example showing how single-pixel values can be recovered
  from box sums using $F = A + B + C + D - E$. }
\label{fig:2dboxes}
\end{figure}
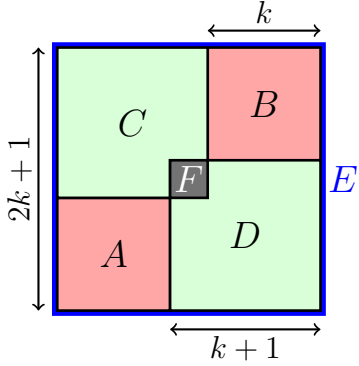

To extend the above to the general case when the box sizes $k_{1},
k_{2}, k_{3}$ are pairwise coprime, observe that, after relabeling if
necessary so that $k_{3}$ is odd, the Chinese remainder theorem
guarantees the existence of a positive integer $k$ and a triplet
$(m_1, m_2, m_3) \in \N^{3}$ such that,
\[
	k= m_1k_{1}, \quad k + 1 = m_2k_{2}, \quad 2k + 1 = m_3k_{3}
\]
Now note that we can recover $y_{m_jk_j}$ from $y_{k_j}$.  This
reduces the problem to the setting illustrated in
Figure~\ref{fig:2dboxes}, where we can use
Eq.~\eqref{eqn:dim_two_special_case} for reconstruction.  Moreover,
the coefficients may be chosen so that $\max(m_{1}, m_{2}, m_{3}) \leq
k_{1}k_{2}k_{3}$. Hence, this method again yields a local
reconstruction algorithm.

\section{Least Squares Reconstruction}
\label{sec:lsqr}

Let $u$ be a finite $d$-dimensional signal (the target) and $z =
\{z_j\}_{j=1}^s$ be a collection of measurements acquired with pixel
sizes $\{k_j\}_{j=1}^s$ and corrupted with noise,
\[
z_j = u \otimes b_{k_j} + \epsilon_j.
\]

Under the i.i.d. Gaussian noise model, the maximum likelihood estimate
of $u$ given the measurements $z$ is the solution to a linear least
squares problem,
\begin{equation}
\hat{u} = \arg\min_x \sum_{j=1}^s ||(x \otimes b_{k_j}) - z_{k_j}||^2.
\end{equation}

Let $T$ be the linear map 
\begin{equation}
T(u) = (u \otimes b_{k_1}, \ldots, u \otimes b_{k_s}).
\label{eqn:T}
\end{equation}
When $T$ is invertible, the estimate $\hat{u}$ is unbiased, and the expected square error
of the solution is (see, e.g. \cite{bjorck-1996}),
\begin{equation}
    \E[\|\hat{u} - u\|^{2}]= \sigma^{2} \tr((T'T)^{-1}).
\label{eqn:lsqerror}
\end{equation}
Here $\sigma^2$ is the variance of the per-pixel sensor noise.

For $d$-dimensional signals, the results in
Section~\ref{sec:stability} show that $d+1$ box sizes that are
pairwise coprime lead to invertible maps $T$ and a stable least
squares reconstruction.

We also consider the use of Tikhonov regularization with a regularization parameter $\lambda$ 
to compute an estimate,
\begin{equation}
\hat{u} = \arg\min_x \sum_{j=1}^s ||(x \otimes b_{k_j}) - z_{k_j}||^2 + \lambda||x||^2.
\end{equation}

Below we discuss how we solve the least squares problem efficiently,
and then in the subsequent section, we consider stability of
(unregularized) least squares reconstruction using different box
sizes.

\subsection{Fourier domain reconstruction}
\label{sec:FFT}

In the cyclic convolution setting, the signal can be reconstructed in
the Fourier domain.  This leads to a very fast method using the FFT
and inverse FFT.  In practice, the method can be directly applied to
measurements collected with a masked sensor or by padding the boundary
of measurements obtained with a regular (non-masked) sensor.

\begin{proposition}
    With cyclic convolution the regularized least squares solution $\hat{u}$ is given by,
    \begin{equation}
    \overline{\hat{u}}(\omega) = \frac{\sum_{j=1}^s \overline{b_{k_j}}(\omega)^{*} \overline{z_j}(\omega)}{\lambda + \sum_{j=1}^s |\overline{b_{k_j}}(\omega)|^{2}}.
    \label{eqn:filter}
    \end{equation}
\label{prop:finv}
\end{proposition}
\begin{proof}
By applying Parseval's theorem and the convolution theorem to the
least squares objective we can express the squared error as,
\begin{align*}
& \sum_{j=1}^s ||(x \otimes b_{k_j}) - z_{k_j}||^2 + \lambda||x||^2 \\ & = \sum_{j=1}^s ||\overline{x} \odot \overline{b_{k_j}} - \overline{z_{k_j}}||^2 + \lambda||\overline{x}||^2 \\
& = \sum_{\omega} \sum_{j=1}^s |\overline{x}(\omega) \overline{b_{k_j}}(\omega) - \overline{z_{k_j}}(\omega)|^2 + \lambda|\overline{x}(\omega)|^2.
\end{align*}
This leads to a separate single-variable quadratic minimization
problem for each frequency component.  A closed form solution for this
quadratic problem can be obtained by setting the derivative to zero,
leading to the statement of the proposition.
\end{proof}

Using Proposition~\ref{prop:finv}, we can see that $\hat{u}$ is a sum
of $s$ different signals $\hat{u}_i$ each estimated from measurements
at a single scale $z_i$ by convolution with a filter $h_i$,
\begin{align*}
&\hat{u} = \sum_{j=1}^s \hat{u}_j, \\
&\hat{u}_j = z_j \otimes h_j, \\
&\overline{h_j}(\omega) = \frac{\overline{b_{k_j}}(\omega)^*}{\lambda + \sum_{j'=1}^s|\overline{b_{k_{j'}}}(\omega)|^{2}}.
\end{align*}

\subsection{Iterative method with fast box-sums}

We have also implemented a fast reconstruction algorithm for the
regularized least squares problem using the LSQR solver
\cite{paige-1982}. The approach can be used to solve the
reconstruction from valid convolution measurements directly, without
padding the data or masking the sensor.

The LSQR method involves repeated applications of the map $T$ and of
the transpose $T'$.  Since $T$ is defined using convolutions with
boxes, both $T$ and $T'$ can be implemented efficiently using integral
images (summed-area tables) or other fast methods for box filtering
(see, e.g.,\cite{mcdonnell-1981}, \cite{crow-1984}).

Each application of $T$ or $T'$ takes time $O(NS)$, where $N$ is the
number of pixels in the signal being reconstructed and $S$ is the
number of scales used for the reconstruction.  In particular, the
running time is linear in $N$ and independent of the box sizes used
for the measurements.

\section{Stability Analysis}
\label{sec:stability}

In this section we study the general case of reconstructing
$d$-dimensional signals.  Key to the analysis is the notion of coprime
(or relatively prime) integers.  Recall that $a$ and $b$ are coprime
if $\gcd(a,b)=1$.

Our main result is that an imaging system with $d+1$ pixel sizes (all
greater than one) has a stable pseudoinverse exactly when the pixel
sizes are pairwise coprime (Theorem~\ref{thm:condition_number}).
Furthermore, we derive an asymptotic formula for the expected squared
error of the least squares reconstruction
(Theorem~\ref{thm:cyclic_trace}).

The key to proving both theorems is to understand the singular values
of the map $T$ in Eq.~\eqref{eqn:T}.  To analyze these singular values
we consider the particular case of cyclic convolutions, first with a
single box (Section~\ref{sec:onebox}) and then with multiple boxes
(Section~\ref{sec:multiplebox}).

Intuitively, we show that, for $d$-dimensional signals, the zeros of
the Fourier transforms of $d+1$ boxes with pairwise coprime sizes do
not have a common intersection.

\subsection{Convolution with one box}
\label{sec:onebox}

For $k \in \N$, let $b_{k} \in \R^{[n]^{d}}$ denote the
$d$-dimensional box filter of size $k$,
\[
    b_{k}(t_{1}, \hdots, t_{d}) = \begin{cases}
        1/k^{d} & \text{if $t_{\ell} \in [k]$ for all $\ell = 1, \hdots, d$,} \\
        0 & \text{otherwise.}
    \end{cases}
\]
In this section, we normalize the box filter to sum to $1$. This
matches an experimental setup where exposure times are varied
according to the pixel size to avoid saturation.

\begin{figure}[b]
    \centering
    \includegraphics[width=0.8\columnwidth]{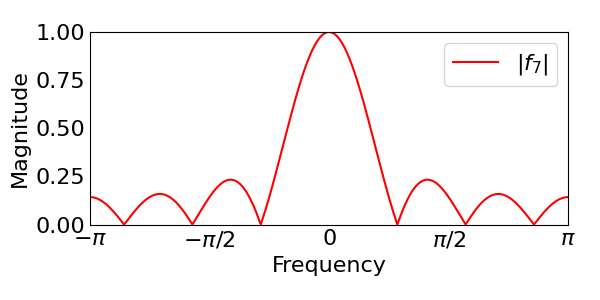}
    \caption{The magnitude of the periodic sinc $f_{7}$.  The discrete
      Fourier transform (DFT) of a box of width $7$ is obtained as a
      discrete sampling of this function.  The maximum of $|f_k|$
      occurs at $\omega = 0$ with value equal to $1$, and the zeros
      occur at non-zero multiples of $2\pi/k$.}
    \label{fig:periodic_sinc_box7}
\end{figure}

Note that the $d$-dimensional box filter is separable, and factors as
the product of $d$ one-dimensional filters,
\begin{equation}
    b_{k}(t_{1}, \hdots, t_{d}) = \prod_{\ell = 1}^{d} b_{k}(t_{\ell}).
\end{equation}

The singular values of the map $u \rightarrow u \otimes b_k$ are
determined by the Fourier transform of $b_k$.  To characterize the
Fourier transform we define the periodic sinc.

\begin{definition}[Periodic sinc]
\[
    f_{k}(\omega) = \frac{\sin(\frac{k\omega}{2})}{k\sin(\frac{\omega}{2})}
\]
for $\omega \in (0, 2\pi)$. An application of L'H\^{o}pital's rule
shows that $f_{k}$ extends continuously to a function on $[0, 2\pi]$
with
\[
    f_{k}(0) = f_{k}(2\pi) = 1.
\]
In what follows, we work with this continuous extension.
\end{definition}

Note that $f_{k}$ has exactly $k-1$ zeros at non-zero integer
multiples of $\frac{2\pi}{k}$.  See
Figure~\ref{fig:periodic_sinc_box7} for an illustration.

The following identity will be useful to compute the discrete Fourier
transform of a box,
\begin{equation}
    f_{k}(\omega) = \frac{e^{i\frac{1}{2}(k - 1) \omega}}{k} \sum_{t \in [k]} e^{-it \omega}.
\label{eqn:identity_for_f}
\end{equation}
The identity follows by expanding the $\sin$ function in terms of
exponentials and using the geometric series formula.

\begin{lemma}[DFT]
\label{lem:fourier_transform_of_box}
For any $(\omega_{1}, \hdots, \omega_{d}) \in \frac{2\pi}{n}[n]^{d}$, 
\begin{equation}
    \overline{b_{k}}(\omega_{1}, \hdots, \omega_{d}) = \prod_{\ell = 1}^{d} e^{-i\frac{1}{2}(k - 1) \omega_{\ell}} f_{k}(\omega_{\ell}).
\end{equation}
\end{lemma}

\begin{proof}
By separability of the $d$-dimensional box, the $d$-dimensional Fourier transform factors: 
\begin{equation}
    \overline{b_{k}}(\omega_{1}, \hdots, \omega_{d}) = \prod_{\ell = 1}^{d} \overline{b_{k}}(\omega_{\ell})
\end{equation}
and so the problem reduces to computing the Fourier transform of the one-dimensional box.

For the one-dimensional case, let $\omega \in \frac{2\pi}{n}[n]$. If
$\omega \neq 0$, starting from the definition of the Fourier
transform, the identity~\eqref{eqn:identity_for_f} yields,
\[
    \overline{b_{k}}(\omega) = \sum_{t \in [n]} e^{-it \omega} b_{k}(t) = \frac{1}{k} \sum_{t \in [k]} e^{-it \omega} = e^{-i\frac{1}{2}(k - 1) \omega} f_{k}(\omega)
\]
Finally, for $\omega = 0$, we can check directly that the left hand
side equals the right hand side.  This proves the claim.
\end{proof}

As a result of Lemma~\ref{lem:fourier_transform_of_box}, the zeros of
$\overline{b_{k}}$ are determined by the zeros of $f_{k}$ and the
discrete frequency sampling in the Fourier transform.  When the
discrete sampling avoids the zeros of $f_{k}$, the Fourier transform
does not vanish, but may still attain arbitrarily small values,
leading to poor conditioning of the imaging system.

The following results make these observations precise.  We note that
$\overline{b_k}$ has no zeros if $k$ and $n$ (the signal length) are
relatively prime.  Nonetheless, for any $k>1$, the number of values in
the Fourier transform of $b_k$ that are \emph{close to zero} is
non-trivial when $n$ is large.

\begin{proposition}
If $\gcd(k, n) = 1$, then $\overline{b_{k}}$ has no zeros.
\end{proposition}

\begin{proof}
By separability, it suffices to treat the one-dimensional case. Let
$\omega \in \frac{2\pi}{n}[n]$. By
Lemma~\ref{lem:fourier_transform_of_box}, $\overline{b_{k}}(\omega) =
0$ if and only if $f_{k}(\omega) = 0$. The zeros of $f_{k}$ are
precisely the set
\begin{equation}
    Z_{k} = \lBc \tfrac{2\pi}{k}m: 1 \leq m \leq k - 1 \rBc.
\label{eqn:one_dim_zero_set}
\end{equation}
Under the coprime condition $\gcd(n, k) = 1$, we can see that the
discrete frequencies $\frac{2\pi}{n}[n]$ do not intersect with
$Z_{k}$.
\end{proof}

In the subsequent proofs, we make use of the following observation,
\begin{equation}
    \max_{\omega \in [0, 2\pi]} |f_{k}(\omega)| = 1.
\label{eqn:trivial_bound_on_f}
\end{equation}
This follows by taking absolute value on both sides
of~\eqref{eqn:identity_for_f} and noting that the maximum is attained
at $\omega = 0$.

A consequence of Eq.~\eqref{eqn:trivial_bound_on_f} is that $|f_{k}|$
is bounded by a constant depending only on the box size. Hence, using
Lemma~\ref{lem:fourier_transform_of_box}, the magnitude of the Fourier
transform $\overline{b_{k}}$ is small whenever \emph{one} coordinate
lies near a zero of $f_{k}$.

\begin{proposition}
For any $\epsilon > 0$, there exists $n_{0} \in \N$ such that, for all
$n \geq n_{0}$, the number of values in $\overline{b_k}$ that are
$\epsilon$-near zero,
\[
\begin{aligned}
    \# \lBc (\omega_{1}, \hdots, \omega_{d}) \in \tfrac{2\pi}{n}[n]^{d}: |\overline{b_{k}}(\omega_{1}, \hdots, \omega_{d})| \leq \epsilon \rBc
\end{aligned}
\]
is at least
$
    d(k - 1)n^{d - 1} - \tbinom{d}{2}(k - 1)^{2}n^{d-2}
$
\label{pro:nearzeros}
\end{proposition}
\begin{proof}
Uniform continuity of $f_{k}$ implies there exists $\delta > 0$ such that, 
\[
    |f_{k}(\omega)| \leq \epsilon
\]
whenever there is a zero $\omega_{0} \in Z_{k}$ with $|\omega -
\omega_{0}| \leq \delta$.

Fix $n \in \N$.  Let $\theta(\omega_0)$ be a value in the discrete set
$\frac{2\pi}{n}[n]$ closest to $\omega_0$ and define $S$,
\[
    S = \{\theta(\omega_{0}): \omega_{0} \in Z_{k}\}.
\]
Choose $n_{0}$ large enough so that $\frac{2\pi}{n_{0}} \leq
\min(\delta, \frac{\pi}{k})$. Provided that $n \geq n_{0}$, the set
$S$ consists of $k - 1$ distinct frequencies, each at distance at most
$\delta$ from some zero of $f_{k}$.

Now consider any $(\omega_{1}, \hdots, \omega_{d}) \in
\frac{2\pi}{n}[n]^{d}$. If any coordinate $\omega_{j}$ lies in $S$,
then
\[
    |\overline{b_{k}}(\omega_{1}, \hdots, \omega_{d})| = \prod_{\ell = 1}^{d} |f_{k}(\omega_{\ell})| \leq |f_{k}(\omega_{j})| \leq \epsilon
\]
Therefore, the number of values in $|\overline{b_k}|$ that are at most
$\epsilon$ is no smaller than the number of discrete frequencies
$(\omega_1, \hdots, \omega_d)$ with some coordinate in $S$. The number
of choices for the latter is
\[
    n^{d} - (n - k + 1)^{d}.
\]
Expanding this polynomial yields the claim.
\end{proof}

Proposition~\ref{pro:nearzeros} shows that cyclic convolution with a
single box does not lead to a stable inverse problem. The Fourier
transform of the box is nearly zero for some frequencies, making the
reconstruction sensitive to noise and numerical roundoff errors.  This
result can be generalized to the linear system formed by combining $s$
measurements taken at different scales, whenever $s \leq d$.

\subsection{Multiscale measurements}
\label{sec:multiplebox}

Now we build on the previous results to establish that, for
$d$-dimensional signals, $d + 1$ convolutions with boxes of different
sizes can lead to an imaging system with stable pseudoinverse.  The
key idea is that when $a$ and $b$ are coprime, the zeros of the
Fourier transforms of one-dimensional boxes of size $a$ and $b$ do not
overlap.  Preventing overlap in $d$-dimensions requires $d+1$ boxes.

For the discussion below, we fix positive integers corresponding to
$d+1$ pixel sizes $\{{k_j}\}_{j=1}^{d+1}$. We shall assume that $k_{j}
\geq 2$ as otherwise the reconstruction problem is not ambiguous. Let
$T$ be the stacked map,
\[
    T(u) = (u \otimes b_{k_1}, \ldots, u \otimes b_{k_{d+1}}).
\]
The following result relates the singular values of a stacked map to
the singular values of the individual maps.

\begin{lemma}
  \label{lem:stack}
  Let $H_1,\ldots,H_s$ be cyclic shift invariant linear maps and
  $\sigma_j(\omega)$ be the singular values of $H_j$.  The singular
  values of the stacked map $H(u) = (H_1(u), \ldots, H_s(u))$ are
  \[\sigma(\omega) = \sqrt{\sum_{j=1}^s \sigma_j^2(\omega)}.\]
\end{lemma}
\begin{proof}
Note that $H^*H = \sum_{j=1}^s H_j^* H_j$ and that all maps in the RHS
are diagonalized by the DFT, and thus share the same eigenvectors,
namely the discrete Fourier modes.

Let $v$ be a Fourier mode and $\lambda_j$ be the corresponding
eigenvalue of $H_j$. Then $\lambda_j^*$ is the corresponding
eigenvalue of $H_j^*$. As $\omega$ ranges over all frequencies, it
follows that $v(\omega)$ is an eigenvector of $H^*H$ with eigenvalue
\[
    \lambda(\omega) = \sum_{j=1}^s |\lambda_j(\omega)|^2.
\]
The result follows from the identities $\sigma(\omega) =
\sqrt{\lambda(\omega)}$ and $\sigma_j(\omega) = |\lambda_j(\omega)|$.
\end{proof}

It will be convenient to define the function
\[
    f(\omega_{1}, \hdots, \omega_{d}) = \lPa \sum_{j = 1}^{d+1} \prod_{\ell = 1}^{d} f_{k_{j}}^{2}(\omega_{\ell}) \rPa^{\frac{1}{2}}
\]
for all $(\omega_{1}, \hdots, \omega_{d}) \in [0, 2\pi]^{d}$.

By Lemmas~\ref{lem:stack} and~\ref{lem:fourier_transform_of_box}, the
singular values of $T$ are the values of $f$ at each discrete
frequency.

\begin{proposition}
If $\{k_j\}_{j = 1}^{d + 1}$ are pairwise coprime, then $f$ is
strictly positive on $[0, 2\pi]^{d}$.
\label{pro:positivity}
\end{proposition}

\begin{proof}
Since $f$ is non-negative and continuous, and $[0, 2\pi]^{d}$ is
compact, it will be sufficient to show that $f$ has no zeros.

Fix $(\omega_{1}, \hdots, \omega_{d}) \in [0, 2\pi]^{d}$. Suppose, for
the sake of a contradiction, that $f(\omega_{1}, \hdots, \omega_{d}) =
0$. By definition,
\[
    \sum_{j = 1}^{d+1} \lPa \prod_{\ell = 1}^{d} f_{k_{j}}^2(\omega_{\ell}) \rPa = 0.
\]
Therefore, for each $k_j$, there is at least one index $\ell_{j}$ such
that $f_{k_{j}}(\omega_{\ell_{j}}) = 0$. Now recall that the zeros of
$f_{k_{j}}$ are the sets
\begin{equation}
Z_{k_{j}} = \lBc \tfrac{2\pi}{k_{j}}m: 1 \leq m \leq k_{j} - 1 \rBc.
\label{eqn:zero_set}
\end{equation}
Since $\{k_{j}\}_{j = 1}^{d+1}$ is pairwise coprime, the zero sets
$\{Z_{k_{j}}\}_{j = 1}^{d+1}$ are pairwise disjoint. Hence each
coordinate $\omega_{\ell}$ belongs to at most one of the sets
$Z_{k_{j}}$ and in particular, the indices $\ell_{j}$ are
distinct. But this is impossible; there are $d$ coordinates and $d +
1$ indices, so at least one of the indices is repeated. This
contradiction shows that $f$ has no zeros.
\end{proof}

The previous results suffices to show that $T$ is invertible when
$\{k_j\}_{j = 1}^{d + 1}$ are pairwise coprime.  The condition number
$\kappa(T)$ measures the relative stability of the pseudoinverse, and
is defined as the ratio of the largest and smallest singular values of
$T$.  Theorem~\ref{thm:condition_number} shows how the coprime
hypothesis is fundamental for stability.

\begin{theorem}
\label{thm:condition_number}
$\;$
\begin{enumerate}
    \item If $\{k_j\}_{j=1}^{d+1}$ are pairwise coprime, \[\exists C>0 \text{ such that } \kappa(T) \leq C \text{ for all } n.\]  
    \item If $\{k_j\}_{j=1}^{d+1}$ are not pairwise coprime \[\lim_{n \rightarrow \infty} \kappa(T) = \infty.\]
\end{enumerate}
\end{theorem}

\begin{proof}
Suppose first that the integers are pairwise coprime.
Using~\eqref{eqn:trivial_bound_on_f}, the largest singular value of $T$ is
\begin{align}
    f(0) = \sqrt{d + 1}.
\end{align}
The smallest singular value of $T$ is at least 
\[
    M = \min_{(\omega_{1}, \hdots, \omega_{d}) \in [0, 2\pi]^{d}} f(\omega_{1}, \hdots, \omega_{d}).
\]
For pairwise coprime box sizes, the minimum $M$ is strictly
positive. The first part of the theorem follows by taking
\[
    C = \frac{\sqrt{d + 1}}{M}.
\]

It remains to prove the second part. As the singular values are a
sampling of the continuous function $f$ on a discrete grid whose
sampling period becomes arbitrarily small, the smallest singular value
of $T$ converges $M$, which is zero if $f$ has a zero.

We show that $f$ has a zero. Consider the zero sets $Z_{k_{j}}$ of
$f_k$ in Eq.~\eqref{eqn:zero_set}. For $k_{j} \geq 2$, these zero sets
are non-empty. Without loss of generality we may assume that $k_1$ and
$k_{d+1}$ are not coprime, and $r = \gcd(k_{1}, k_{d + 1}) \geq
2$. Let
\[
    \omega_{1} = 2\pi/r \in Z_{k_{1}} \cap Z_{k_{d + 1}}
\]
and choose $\omega_{j} \in Z_{k_{j}}$ for $j \neq 1$. We can readily
verify that this choice leads to $f(\omega_{1}, \hdots, \omega_{d})=0$
\end{proof}

\subsection{Expected least squares error} \label{subsec:expected_LSE}

Finally we consider the error of the least squares reconstruction.
The quantity $\tr((T'T)^{-1})/n^d$ is of key interest, because it
determines the expected mean squared error.  The next result gives an
asymptotic formula for large signals.

\begin{theorem}
\label{thm:cyclic_trace}
If $\{k_j\}_{j = 1}^{d + 1}$ are pairwise coprime, 
\[
    \lim_{n \rightarrow \infty} \frac{\tr((T'T)^{-1})}{n^{d}} = \frac{1}{(2\pi)^{d}} \int_{[0, 2\pi]^{d}} \frac{\dx \omega_{1} \cdots \dx \omega_{d}}{f^{2}(\omega_{1}, \hdots, \omega_{d})}.
\]
\end{theorem}

\begin{proof}
Expanding the trace as a sum of eigenvalues gives, 
\begin{align}
    \tr((T'T)^{-1}) = \sum_{(\omega_{1}, \hdots, \omega_{d}) \in \frac{2\pi}{n}[n]^{d}} \frac{1}{f(\omega_{1}, \hdots, \omega_{d})^{2}}.
\end{align}
Dividing by $n^{d}$, the RHS of the above display becomes a Riemann
sum on $[0, 2\pi]^{d}$. Since $f$ is continuous and strictly positive,
this Riemann sum converges.
\end{proof}

\begin{figure}[b]
    \centering
    \includegraphics[width=0.8\columnwidth]{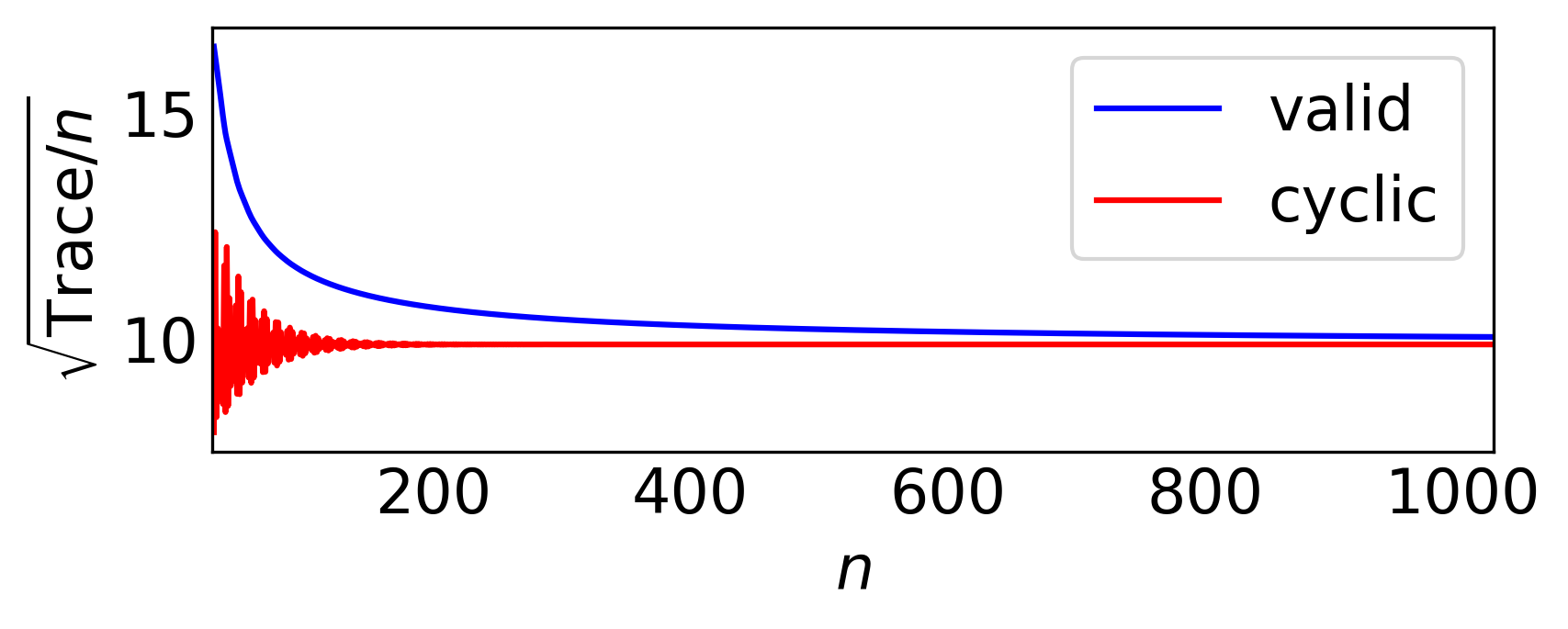}
    \caption{Plot of $(\tr((T'T)^{-1})/n)^{\frac{1}{2}}$ illustrating convergence in the one-dimensional case using box sizes $9$ and $11$.}
    \label{fig:traces}
\end{figure}

Figure~\ref{fig:traces} shows two plots of
$(\tr((T'T)^{-1})/n)^{\frac{1}{2}}$ as a function of $n$ for
one-dimensional signals with box sizes $9$ and $11$.  We plot the
square root of the normalized trace because it better corresponds to
an RMS value for a reconstruction with least squares.  We consider
maps defined via cyclic or valid convolution separately.  The trace in
each case can be computed numerically via explicit construction of an
imaging matrix.  We can see in the plots that the mean squared error
obtained with measurements defined by valid convolution approaches the
mean squared error obtained with measurements defined by cyclic
convolutions.

\subsection{Noise-resolution tradeoff}

Using the previous results we can see a fundamental tradeoff between
the reconstruction error and sensor resolution when the per-pixel
measurement noise is fixed.

Using Eq.~\ref{eqn:identity_for_f}, notice that
\[
    \int_{[0, 2\pi]} f_{k_{j}}^{2}(\omega) \, \dx{\omega} = \frac{2\pi}{k_{j}}.
\]
Letting $k = \min k_{j}$, an application of Jensen's inequality gives the lower bound, 
\[
    \frac{1}{(2\pi)^{d}} \int_{[0, 2\pi]^{d}} \frac{\dx \omega_{1} \cdots \dx \omega_{d}}{f^{2}(\omega_{1}, \hdots, \omega_{d})} \geq \frac{k^{d}}{d + 1}.
\]

As a consequence of this bound we can make the following observation.
\begin{observation}
\label{obs:tradeoff}
With pixels of size at least $k$ the expected mean squared error is at
least $k^{d}/(d+1)$ times larger than obtained using single pixel
measurements.
\end{observation}

\section{Experimental Results}
\label{sec:experiments}

\subsection{Setup}

To systematically explore multiscale reconstructions over a large
range of effective scales, we performed experiments using a CCD
camera. CCDs are a helpful testbed for super resolution, because the
photoelectrons accumulated on each pixel can be binned (further
accumulated) in hardware prior to readout. Here, we will exploit
binning to enable rapid measurement over a range of both 1D and 2D
pixel sizes. In particular, the data was acquired using a monochrome 1
megapixel (1024x1024) CCD camera (Teledyne Princeton Instruments PIXIS
1024B) with 13 $\mu$m square pixels in combination with a 25-mm, f/1.8
lens (Newyi) opened to maximum aperture and focused at approximately
0.5 meters.  This optical system ensures that the diffraction-limited
optical image resolution at the sensor plane is smaller than all pixel
sizes used in the following experiments.

For 1D experiments, rectangular bins of size $k \times 1$ were used to
accumulate sets of $k$ adjacent pixels in a single column. No binning
occurred across the rows of the sensor. For 2D experiments, square
bins of size $k \times k$ were used to accumulate sets of $k^2$
pixels. This operation occurs before the signal is amplified and
readout. Furthermore, to ensure that we are using the full dynamic
range of the sensor, we scale the exposure time for each recorded
image (by $1/k$ in 1D experiments and $1/k^2$ in 2D experiments) to
prevent saturation and ensure that the approximate range of measured
values is similar for different bin sizes. This also ensures that the
total measurement time for each bin size (i.e. including shifted
images) is approximately the same as the measurement of the entire
sensor with single-pixel bins.

Obtaining interlaced measurements that span the bins defining a valid
convolution requires capturing $k$ distinct images in the 1D case (and
likewise, $k^2$ in the 2D case) with shifted bins. For example, the
measurements for 1D bins of size $k=4$ involve origin shifts of 0, 1,
2, and 3 pixels, which for our 1024-pixel width sensors returns rows
containing 256, 255, 255, and 255 binned counts, respectively. These
1021 data points span the full range of the valid convolution: $n-k+1$
= 1021 data points. In physical terms, note that one-pixel shifts on
our Pixis 1024B CCD Camera correspond to 13 $\mu$m. Adjustments of the
binning patterns and image recording were automated using Python
scripts and the LightField\textsuperscript{\textregistered} camera
software.

Image targets were printed using a standard laser printer on white
paper. For each set of experiments, background measurements were taken
with exposure times and binning regions identical to the data, but
with the shutter closed. For completeness, flatfield calibration
measurements were also acquired for each bin size by imaging a blank
sheet of paper. To show that the results here are robust with or
without flatfield correction, the 1D results presented here are only
background correct, whereas the 2D results are both background and
flatfield corrected. All targets were placed on a static platform and
illuminated by an LED. The illumination intensity was adjusted to
ensure that the highest light levels were below the saturation range
of the 16-bit readout electronics.

\subsection{1D Results and Noise Analysis}
\label{subsec:1D_results}

\begin{figure}[b!]
\begin{subfigure}{1\columnwidth}
\includegraphics[width=\textwidth]{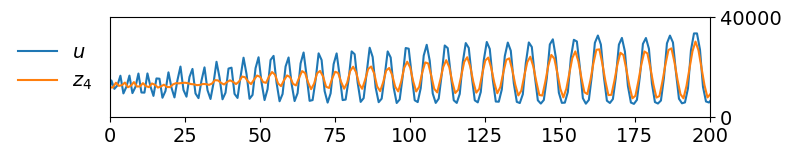}
\includegraphics[width=\textwidth]{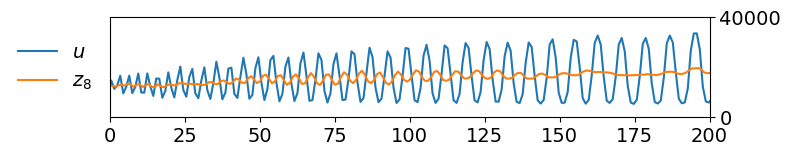}
\includegraphics[width=\textwidth]{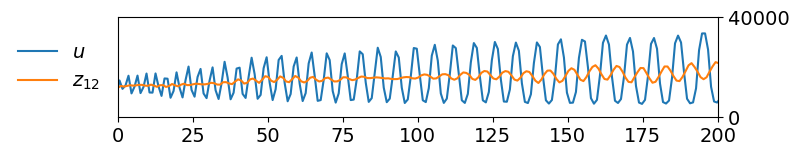}
\includegraphics[width=\textwidth]{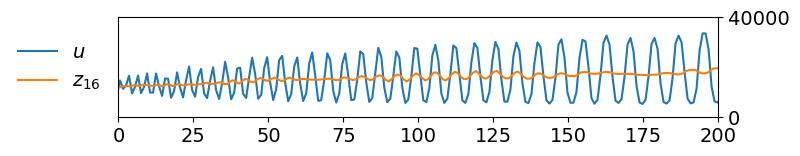}
\includegraphics[width=\textwidth]{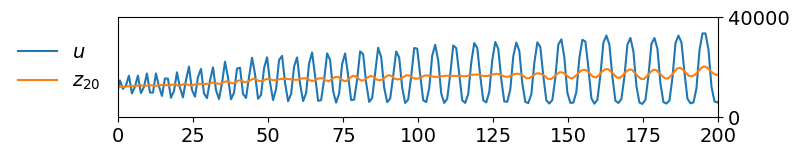}
\includegraphics[width=\textwidth]{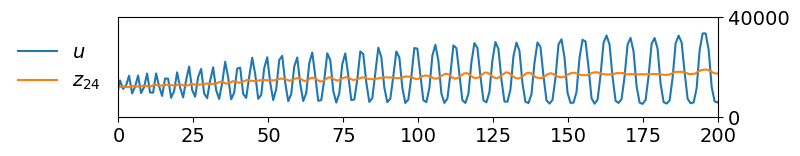}
\end{subfigure}
\caption{Experimental measurements of 1D linearly increasing grating
  for bin sizes 4, 8, 16, 20, and 24. Each binned measurement (orange
  line) is plotted alongside the single-pixel measurement (blue line)
  so that the diminishing contrast and aliasing can be directly
  observed.}
\label{fig:Grating_Box_Measurements}
\end{figure}

To perform experiments with``one-dimensional" signals, we imaged a
printed target consisting of a sequence of horizontal bars with
linearly increasing size and separation. Binning was performed along
the sensor columns, which were aligned perpendicular to the target
bars to see maximum variation. In the horizontal direction (along any
given row of the sensor), the target was approximately constant.

\begin{figure*}[t!]
\begin{subfigure}{1\columnwidth}
\caption{Bin Sizes 12 and 15 (GCD = 3)}
\includegraphics[width=\textwidth]{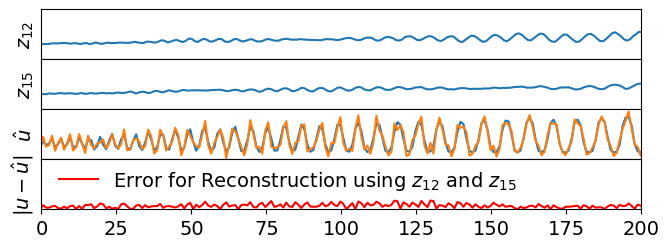}
\caption{Bin Sizes 12 and 16 (GCD = 4)}
\includegraphics[width=\textwidth]{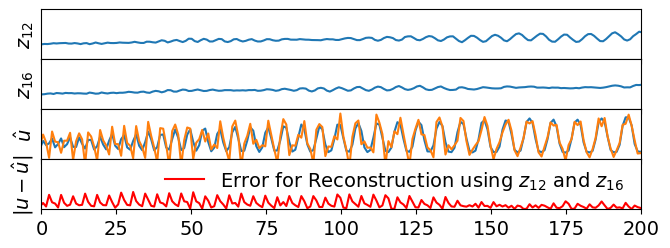}
\end{subfigure}
\begin{subfigure}{1\columnwidth}
\caption{Bin Sizes 13 and 15 (Coprime)}
\includegraphics[width=\textwidth]{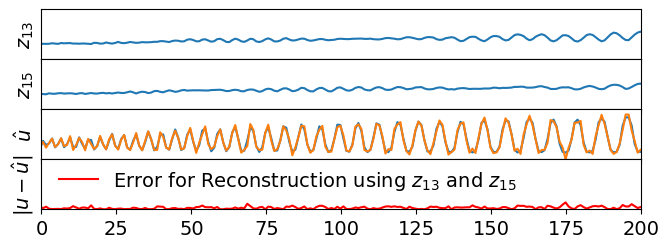}
\caption{Bin Sizes 13 and 16 (Coprime)}
\includegraphics[width=\textwidth]{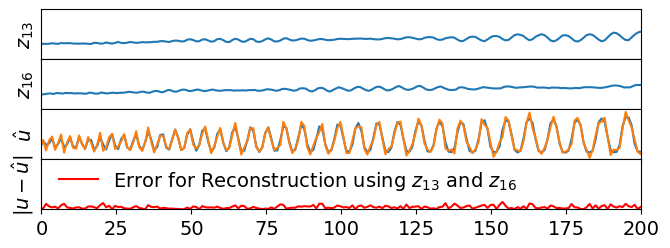}
\end{subfigure}
    \centering
    \caption{Example 1D reconstructions using adjacent pairs of bin
      sizes 12/13 (left/right) and 15/16 (top/down). When the bin
      sizes share a common factor (a and b), periodic errors
      arise. However, when bin sizes are coprime (c and d),
      reconstruction errors are low even when using a larger bin sizes
      (such as 13 in place of 12).}
    \label{fig:1d_reconstructions}
\end{figure*}

Data acquired from each sensor column could thus be treated as
independent measurement sets enabling 1024 experiments for one
acquisition. Figure~\ref{fig:Grating_Box_Measurements} shows examples
of measured data ($z_k$) with different bin sizes for one column
alongside a ground truth estimate ($u$) obtained by averaging twenty
single-pixel measurements of that same column. Note that, as the bin
size grows, the binned measurements exhibit diminishing contrast. In
addition, aliasing can be readily observed, such as in $z_8$ and
$z_{12}$ where the binned measurements are out-of-phase with $u$.

For these 1D experiments, we computed reconstructions by explicitly
constructing imaging matrices for valid convolutions and using the
\texttt{numpy} solver \texttt{numpy.linalg.lstq} which computes a
minimum norm reconstruction when the least squares solution is not
unique.

\begin{figure}
\begin{subfigure}{0.45\textwidth}
\caption{Predicted $(\E[\|\hat{u} - u\|^{2}]/n\sigma^2)^\frac{1}{2} = (\tr((T'T)^{-1})/n)^\frac{1}{2}$}
\centering
\vspace{.1cm}
\includegraphics[height=0.7\textwidth]{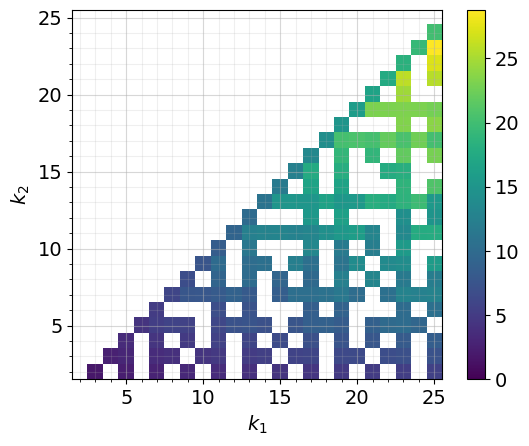}
\end{subfigure}
\begin{subfigure}{0.45\textwidth}
\caption{Experimentally Observed $\left(\overline{||\hat{u}-u||^2}/n\sigma^2\right)^{\frac{1}{2}}$}
\centering
\includegraphics[height=0.7\textwidth]{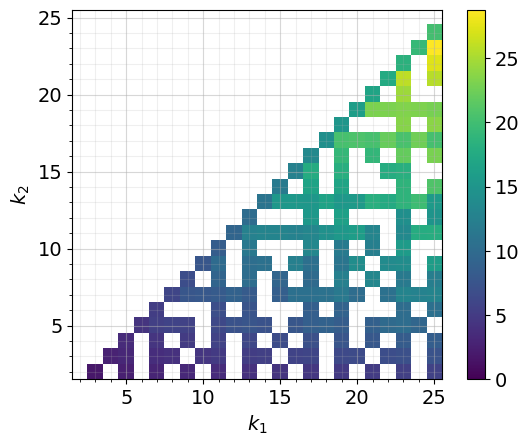}
\end{subfigure}
\caption{Comparison of (a) predicted reconstruction error under i.i.d. Noise with (b) experimentally observed error when using coprime box sizes, $k_1$ and $k_2$.}%
\label{fig:Trace_and_Error}
\end{figure}

Figure~\ref{fig:1d_reconstructions} shows representative 1D
reconstructions obtained from binned measurements. To highlight the
importance of coprime bin sizes, we consider the four reconstructions
created from two sets of adjacent sizes: 12/13 and 15/16. Note that,
when the sizes share a common factor, periodic errors arise. For
example, the reconstruction from $z_{12}$/$z_{15}$
(Figure~\ref{fig:1d_reconstructions}a) shows periodic errors that
originate from the common factor of 3. Such errors are more prominent
for the case of $z_{12}$/$z_{16}$
(Figure~\ref{fig:1d_reconstructions}b). However, when the sizes used
in the reconstruction are relatively prime, as is the case with
$z_{13}$/$z_{15}$ (Figure~\ref{fig:1d_reconstructions}c) and
$z_{13}$/$z_{16}$ (Figure~\ref{fig:1d_reconstructions}d), the original
signal can be reconstructed with low error -- even when using a larger
bin size (13 as opposed to 12) for reconstruction.

\begin{figure*}
  \captionsetup{justification=centering}
  \centering
  \begin{subfigure}{0.18\textwidth}
    \centering
    \caption{130 x 130 $\mu$m\\Pixel Images}
    \includegraphics[height=1\textwidth]{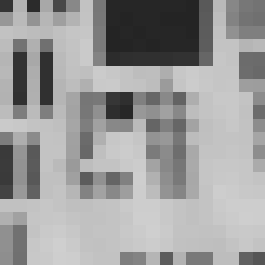}
  \end{subfigure}
  \begin{subfigure}{0.18\textwidth}
    \centering
    \caption{143 x 143 $\mu$m\\Pixel Images}
    \includegraphics[height=1\textwidth]{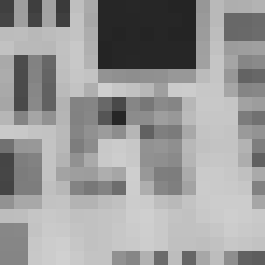}
  \end{subfigure}
  \begin{subfigure}{0.18\textwidth}
    \centering
    \caption{169 x 169 $\mu$m\\Pixel Images}
    \includegraphics[height=1\textwidth]{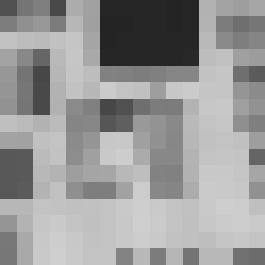}
    \end{subfigure}
  \begin{subfigure}{0.18\textwidth}
    \centering
    \caption{Multiscale Reconstructions}
    \includegraphics[height=1\textwidth]{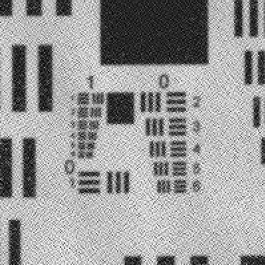}
    \end{subfigure}
    \begin{subfigure}{0.18\textwidth}
    \centering
    \caption{13 x 13 $\mu$m\\Pixel Images}
    \includegraphics[height=1\textwidth]{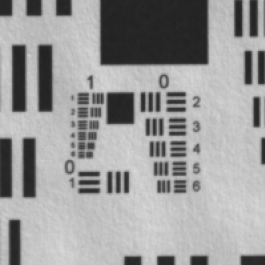}
    \end{subfigure}
 
  \vspace{0.2cm}

  \begin{subfigure}{0.18\textwidth}
    \centering
    \includegraphics[height=1\textwidth]{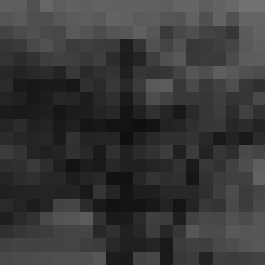} 
  \end{subfigure}
  \begin{subfigure}{0.18\textwidth}
    \centering
    \includegraphics[height=1\textwidth]{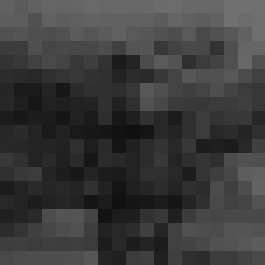} 
  \end{subfigure}
  \begin{subfigure}{0.18\textwidth}
    \centering
    \includegraphics[height=1\textwidth]{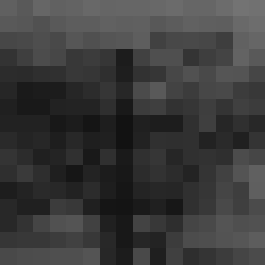} 
  \end{subfigure}
  \begin{subfigure}{0.18\textwidth}
    \centering
    \includegraphics[height=1\textwidth]{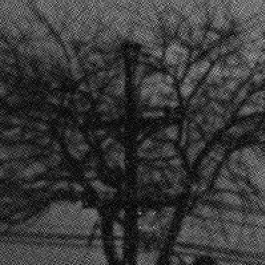}
    \end{subfigure}
    \begin{subfigure}{0.18\textwidth}
    \centering
    \includegraphics[height=1\textwidth]
    {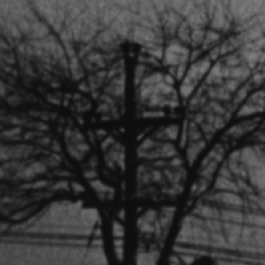}
    \end{subfigure}
  
  \vspace{0.2cm}

  \begin{subfigure}{0.18\textwidth}
    \centering
    \includegraphics[height=1\textwidth]{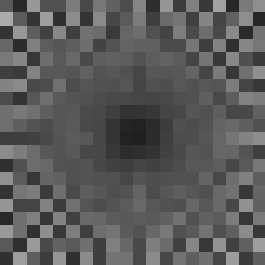} 
  \end{subfigure}
  \begin{subfigure}{0.18\textwidth}
    \centering
    \includegraphics[height=1\textwidth]{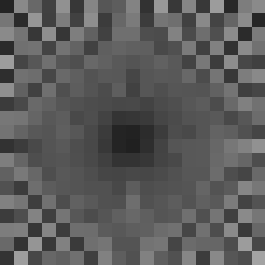} 
  \end{subfigure}
  \begin{subfigure}{0.18\textwidth}
    \centering
    \includegraphics[height=1\textwidth]{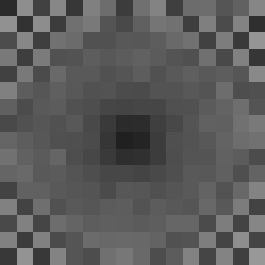} 
  \end{subfigure}
  \begin{subfigure}{0.18\textwidth}
    \centering
    \includegraphics[height=1\textwidth]{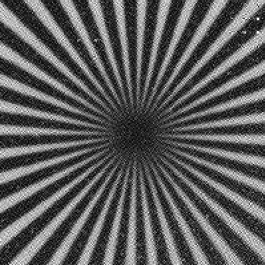}
    \end{subfigure}
    \begin{subfigure}{0.18\textwidth}
    \centering
    \includegraphics[height=1\textwidth]{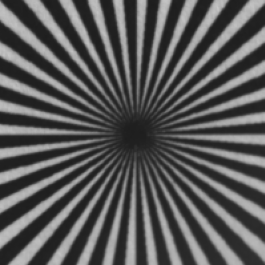}
    \end{subfigure}

\vspace{-0.2cm}
\caption{Three experimental demonstrations of mutliscale
  super-resolution using CCD binning at three image scales.}
\label{fig:2D_Reconstruction_Examples}
\end{figure*}

We also performed a systematic comparison of the empirical
reconstruction errors for coprime bin sizes with the expected errors
of the least squares estimate under i.i.d Gaussian noise. Since each
column of the image sensor represents an independent experiment, we
average the empirical squared error over the 1024 experiments, and
divide by $n$ and the single-pixel measurement variance $\sigma^2$ to
examine per-pixel effects.  Figure~\ref{fig:Trace_and_Error} shows
that the expected squared error predicted by Eq.~\ref{eqn:lsqerror}
(Figure~\ref{fig:Trace_and_Error}a) corresponds very well to the
experimentally observed error (Figure~\ref{fig:Trace_and_Error}b).  We
specifically look at the square-root of the two quantities as they
better correspond to RMS values and are in the scale of the
signal. The ratio of the two quantities plotted here is nearly uniform
(min: 0.99, max: 1.03) for all coprime bin sizes from 2 through
25. The scalebars in Figure~\ref{fig:Trace_and_Error} also serve to
highlight the inherent tradeoff between noise and resolution. Note
that the predicted and experimentally observed error scales linearly
with the resolution enhancement in 1D; that is, using data from bins
of size 10 and 11 to reconstruct single-pixel data yields a noise
level that is 10 times larger than direct measurement at single-pixel
scale. However, at the same time, larger effective pixel measurements
can be used to acquire data over a larger field of view.

\subsection{2D Results and Fourier Domain Reconstructions}

Having compared experiment and theory in the 1D case, we proceed to 2D
experiments and demonstrate examples for which a greater than 10x
resolution enhancement can be readily observed. For these experiments,
we imaged several samples including a United States Air Force
resolution target, an analog film photograph of a tree near a
powerline, and a binary pinwheel.

In 2D, three measurement scales are sufficient to solve the
super-resolution problem without an image
prior. Figure~\ref{fig:2D_Reconstruction_Examples} shows coarse images
of three samples using 10x10, 11x11 and 13x13 bins, which correspond
to effective pixel sizes of 130x130$\mu$m
(Figure~\ref{fig:2D_Reconstruction_Examples}a), 143x143$\mu$m
(Figure~\ref{fig:2D_Reconstruction_Examples}b), and 169x169$\mu$m
(Figure~\ref{fig:2D_Reconstruction_Examples}c), respectively. These
coarse images are unable to resolve any fine features in the three
samples. For the Air Force resolution target and the photograph, the
features are severely blurred, while coarse imaging of the pinwheel
leads to large checkerboard-style aliasing patterns. Nevertheless,
pixel-shifted images at these three scales enable high-resolution
reconstruction (Figure~\ref{fig:2D_Reconstruction_Examples}d) that are
comparable to direct imaging with 13x13$\mu$m single-pixel
measurements (Figure~\ref{fig:2D_Reconstruction_Examples}e).

\begin{figure}
  \captionsetup{justification=centering}
  \centering
  \begin{subfigure}{0.32\columnwidth}
    \centering
    \caption{Single Scale 117~$\mu$m Reconstructions}
    \includegraphics[height=1\textwidth]{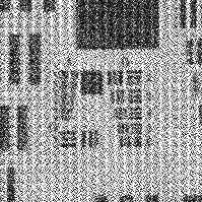}
    \vspace{-1cm}\caption*{\textcolor{black}{\bf 0.435 RMSE}}
  \end{subfigure}
  \begin{subfigure}{0.32\columnwidth}
    \centering
    \caption{Two Scale 117/130~$\mu$m Reconstructions}\includegraphics[height=1\textwidth]{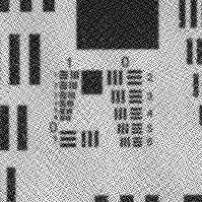}
    \vspace{-1cm}\caption*{\textcolor{black}{\bf 0.229 RMSE}}
  \end{subfigure}
  \begin{subfigure}{0.32\columnwidth}
    \centering
    \vspace{-0.5cm}\caption{Three Scale 117/130/143~$\mu$m
    Reconstructions}
    \includegraphics[height=1\textwidth]{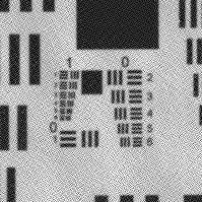} 
    \vspace{-1cm}\caption*{\textcolor{black}{\bf 0.160 RMSE}}
  \end{subfigure}

  \vspace{0.2cm}

    \begin{subfigure}{0.32\columnwidth}
    \centering
    \includegraphics[height=1\textwidth]{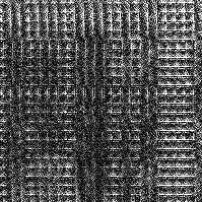} 
   \vspace{-1cm}\caption*{\textcolor{white}{\bf 0.347 RMSE}}
  \end{subfigure}
  \begin{subfigure}{0.32\columnwidth}
    \centering
    \includegraphics[height=1\textwidth]{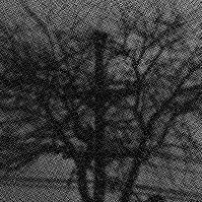} 
  \vspace{-1cm}\caption*{\textcolor{white}{\bf 0.134 RMSE}}
  \end{subfigure}
  \begin{subfigure}{0.32\columnwidth}
    \centering
    \includegraphics[height=1\textwidth]{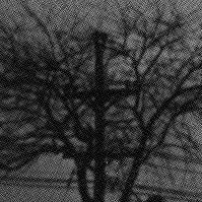}
    \vspace{-1cm}\caption*{\textcolor{white}{\bf 0.094 RMSE} }
  \end{subfigure}

  \vspace{0.2cm}

  \begin{subfigure}{0.32\columnwidth}
    \centering
    \includegraphics[height=1\textwidth]{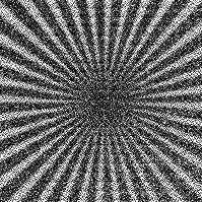} 
  \vspace{-1cm}\caption*{\textcolor{white}{\bf 0.325 RMSE}}
  \end{subfigure}
  \begin{subfigure}{0.32\columnwidth}
    \centering
    \includegraphics[height=1\textwidth]{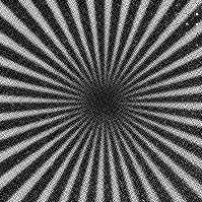} 
\vspace{-1cm}\caption*{\textcolor{white}{\bf 0.183 RMSE}}
\end{subfigure}
  \begin{subfigure}{0.32\columnwidth}
    \centering
    \includegraphics[height=1\textwidth]{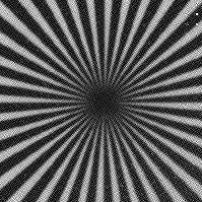} 
    \vspace{-1cm}\caption*{\textcolor{white}{\bf 0.128 RMSE}}
  \end{subfigure}
  
\vspace{-0cm}
\captionsetup{justification=raggedright}
\caption{Experimental demonstration of how using multiple imaging
  scales improves super-resolution reconstructions.}
\label{fig:2D_Multiscale_Benefits}
\end{figure}

The reconstructions shown in
Figure~\ref{fig:2D_Reconstruction_Examples}d were obtained using the
Fourier domain method described in Section~\ref{sec:FFT}, using the
two-dimensional FFT algorithm in \texttt{numpy}.  Reconstruction in
each case took less than 1 second for a 1024x1024 image on a standard
desktop computer with an Intel 2.5Ghz 16 core i5 CPU and 16GB RAM.

Using regularization in the Fourier domain, we can compare
reconstructions that use fewer than the minimum number of measurement
scales required to resolve all ambiguities.
Figure~\ref{fig:2D_Multiscale_Benefits} showcases 2D reconstructions
using one, two, and three measurement scales, all with a small
regularization factor of $\lambda = 10^{-6}$. Here, we acquired data
using 9x9, 10x10, and 11x11 bins, which correspond to effective pixel
sizes of 117x117$\mu$m, 130x130$\mu$m, and 143x143$\mu$m,
respectively. When only a single scale is used
(Figure~\ref{fig:2D_Multiscale_Benefits}a), reconstruction with
regularization allows recovery of some features in the samples, but
exhibits significant periodic errors related to blindspots of the
single box filter. The RMS error of the single-scale reconstruction
can be reduced in half using a second scale
(Figure~\ref{fig:2D_Multiscale_Benefits}b), but as described in
Section~\ref{subsec:2D_box_packing}, two box filters have common
blindspots in 2D and the result still exhibits periodic
artifacts. Finally though, by using a third scale as shown in
(Figure~\ref{fig:2D_Multiscale_Benefits}c), we are both able to
further reduce the RMS error by approximately one third, and most
importantly, suppress the periodic errors that arise from common
blindposts of any pair of box filters. By using three coprime bin
sizes, we have recovered information across all Fourier components in
the high-resolution reconstruction. Of course, there is still noise in
the three scale reconstruction shown in
Figure~\ref{fig:2D_Multiscale_Benefits}c, just as the three-scale
reconstructions in Figure~\ref{fig:2D_Reconstruction_Examples}d have
greater noise than measuring directly at the scale shown in
Figure~\ref{fig:2D_Reconstruction_Examples}e. As discussed in
Section~\ref{subsec:expected_LSE} and experimentally explored in
Section~\ref{subsec:1D_results}, the errors increase with the
resolution enhancement factor, which can be seen here in
2D. Nevertheless, we see that large resolution enhancement factors of
10x or more (e.g., as shown earlier by the 17x improvement in
Figure~\ref{fig:200micronDemo}) can be achieved.

\section{Conclusion}

We have shown that super-resolution under translation is well posed
when a sufficient number of effective resolutions are used to observe
a signal, and validated these results with both one- and
two-dimensional experiments.  The analysis presented here is quite
general and applies to problems beyond two dimensions.  For example,
the results can be applied to three-dimensional signals in
hyperspectral imaging (2D spatial + 1D spectral dimensions).

Existing cameras with pixel-shift technology can already generate the
kinds of measurements needed for super resolution under translation.
Combining pixel-shift technology with variations in optical
magnification using a zoom lens could enable wide-field
super-resolution imaging with high effective magnification. For
example, data collected with three different optical magnifications of
$(100/9)$x, $(100/10)$x and $(100/11)$x can be used to reconstruct an
image at $100$x magnification using only imaging constraints. In
addition, the resulting image has the benefit of having a much larger
field of view than would have been acquired using the same camera with
100x magnification.

Many mobile phones already include three distinct cameras designed to
record images at different magnifications, and multicamera systems
with optical zoom are common in high-end drones and robots.  Many of
these systems naturally move, and can collect data that may be
suitable for super resolution using burst photography (see,
e.g. \cite{handheld}).

There is also room for significant simplifications of the data
acquisition process used here.  Indeed, the reconstruction problem
using $d+1$ valid convolutions is significantly overdetermined.  In
one dimension, we have confirmed that if $k_1 < k_2$ are relatively
prime, then a single image captured with pixel size $k_2$ is
sufficient for super-resolution imaging without priors when combined
with interlaced data captured with pixel size $k_1$.  That is, we only
need to translate one sensor and take a single image with the other.

There is something fundamentally different between using an image
prior to fill in missing detail and reconstruction using imaging
constraints alone.  The fact that we have identified one particular
setting where multiple effective resolutions eliminate ambiguities
suggest more generally, that unstructured images, at a variety of
scales and locations, may enable invertible systems and stable
reconstructions.


\bibliographystyle{IEEEtran}
\bibliography{main}

\end{document}